\documentclass{article}

\usepackage[margin=1.3in]{geometry}
\usepackage{graphicx}
\usepackage[square,numbers]{natbib}
\usepackage[colorlinks=true]{hyperref}
\usepackage{amsmath}
\usepackage{amsfonts}
\usepackage{mathrsfs}
\usepackage{soul}
\usepackage{scabby}
\usepackage{comment}
\usepackage{kbordermatrix}
\usepackage{enumitem}
\usepackage{hhline}
\usepackage{tabularx}
\usepackage{multirow}
\usepackage{booktabs}
\usepackage{tablefootnote}
\usepackage{lipsum}
\usepackage[normalem]{ulem} 

% ...
\title{Estimating Mixture Models via Mixtures of Polynomials}

\author{%
Sida I. Wang \ \ \ \
Arun Tejasvi Chaganty \ \ \ \
Percy Liang\ \ \ \  \\
Computer Science Department,
Stanford University\\%
\texttt{\{sidaw,chaganty,pliang\}@cs.stanford.edu}
}
\date{}
\newif\ifshowcomments
\ifshowcomments
\newcommand\ac[1]{\textcolor{red}{[AC: #1]}}
\newcommand\pl[1]{\textcolor{red}{[PL: #1]}}
\newcommand\pled[1]{\textcolor{orange}{[PL(done): #1]}}
\newcommand\sidaw[1]{\textcolor{green}{[sidaw: #1]}}
\else
\newcommand\ac[1]{}
\newcommand\pl[1]{}
\newcommand\pled[1]{}
\newcommand\sidaw[1]{}
\fi

\providecommand{\bigcdot}{\bullet}

\providecommand{\ind}[1]{_{#1}} 
\providecommand{\inds}[2]{_{#1,#2}} 

\providecommand{\byt}{{\by^*}} 

\providecommand{\balpha}{{\boldsymbol\alpha}} 
\providecommand{\bbeta}{{\boldsymbol\beta}} 
\providecommand{\bgamma}{{\boldsymbol\gamma}} 

\renewcommand{\Re}{\mathbb{R}} 
\providecommand{\NN}{\mathbb{N}} 
\providecommand{\Loy}{\mathscr{L}_{\by}} 
\providecommand{\Mea}{\mathcal{M}} 
\providecommand{\MM}{\mathbf{M}} 
\providecommand{\coeff}[2]{a_{#1#2}}

\providecommand{\sprm}{\theta} 
\providecommand{\sprmt}{\sprm^*} 
\providecommand{\sprmh}{\hat{\sprm}} 

\providecommand{\prm}{{\boldsymbol\sprm}} 
\providecommand{\prmt}{\prm^*} 

\providecommand{\prmh}{\hat{\prm}} 
\providecommand{\allprm}{[\prm_k]_{k=1}^K} 
\providecommand{\allprmt}{[\prmt_k]_{k=1}^K} 
\providecommand{\dat}{\bx} 
\providecommand{\sdat}{x} 

\providecommand{\obs}{\phi} 
\providecommand{\mdeg}{r} 
\providecommand{\degree}[1]{\left|#1\right|}
\providecommand{\smono}[1]{{\left[#1\right]}}
\providecommand{\bUbasis}{\bU\smono{\bbeta_{1}; \ldots; \bbeta_{K}}}
\providecommand{\bUbasisshift}[1]{\bU\smono{\bbeta_{1} + \bgamma_{#1}; \ldots; \bbeta_{K} + \bgamma_{#1}}}

\providecommand{\monos}{\bv} 

\providecommand{\cnstrfun}{g} 
\providecommand{\cnstrfuns}{\boldsymbol{g}} %
\providecommand{\cnstrfunname}{constraint function} %

\providecommand{\XX}{\boldsymbol{X}} 

\providecommand{\Multinomial}{\operatorname{Multinomial}}

\providecommand{\polymom}{Polymom} 
\providecommand{\GeMM}{Gen.MM} 
\providecommand{\tensordecomp}{tensor decomposition}
\providecommand{\tensorstruct}{special tensor structure}
\providecommand{\obsname}{observation function}
\providecommand{\feasibility}[2]{
  \begin{array}{cl}
    \textnormal{find} &{#1} \\
    \textnormal{s.t.}  &#2
  \end{array}
}

\providecommand{\minprob}[3]{
  \begin{array}{cl}
    \underset{#1}{\operatorname{minimize}}&{#2} \\
    \textnormal{s.t.}  &#3
  \end{array}
}

\providecommand{\tr}{\operatorname{tr}} 
\providecommand{\diag}{\operatorname{diag}}
\providecommand{\rank}{\operatorname{rank}}

\providecommand{\T}{\mathsf{T}} 
\providecommand{\vecs}{\operatorname{vecs}}
\providecommand{\colsp}{\operatorname{col}}

\providecommand{\inv}{{-1}}

\newcommand{\EE}{\mathbb{E}}
\newcommand{\indicator}[1]{\operatorname{\boldsymbol{1}}_{#1}}
\newcommand{\convergep}{\overset{p}{\to}}

\providecommand{\given}{\vert} 

\forcecommand[2]{\editnote}{\ensuremath{\underline{\textrm{#1}}_{\textrm{#2}}}}
\forcecommand[1]{\editupnote}{\ensuremath{^{[\hl{\textrm{#1}}]}}}
\forcecommand{\vague}{\editupnote{vague}}
\forcecommand{\verify}{\editupnote{verify}}
\forcecommand{\reword}{\editupnote{reword}}
\forcecommand{\define}{\editupnote{define}}
\forcecommand{\expand}{\editupnote{expand}}
\forcecommand{\needcite}{\editupnote{citation needed}}
\forcecommand[1]{\findcite}{\editupnote{find: #1}}

\forcecommand[3]{\itemref}{\hyperref[#2:#3]{#1~\ref{#2:#3}}}
\forcecommand[1]{\sectionref}{\itemref{Section}{sec}{#1}}
\forcecommand[1]{\appendixref}{\itemref{Appendix}{sec}{#1}}
\forcecommand[1]{\algorithmref}{\itemref{Algorithm}{algo}{#1}}
\forcecommand[1]{\figureref}{\itemref{Figure}{fig}{#1}}
\forcecommand[1]{\tableref}{\itemref{Table}{tbl}{#1}}
\forcecommand[1]{\equationref}{\hyperref[eqn:#1]{Equation \ref{eqn:#1}}}
\forcecommand[1]{\problemref}{\hyperref[prob:#1]{Problem \ref{prob:#1}}}
\forcecommand[1]{\exampleref}{\hyperref[exa:#1]{Example \ref{exa:#1}}}

\declaretheoremstyle[%
  spaceabove=0em,%
  spacebelow=-10em,%
  headfont=\normalfont\itshape,%
  postheadspace=1em%
]{mystyle} 
\declaretheorem[name={Proof},style=mystyle,unnumbered,
]{prf}

\newtheorem{thm}{Theorem}[section]
\newtheorem{lem}[thm]{Lemma}

\newtheorem{prop}[thm]{Proposition}

\newtheorem{example}[thm]{Example}

\renewenvironment{abstract}
 {\small
  \begin{center}
  \bfseries \abstractname
  \end{center}
  \list{}{
    \setlength{\leftmargin}{1.5cm}%
    \setlength{\rightmargin}{\leftmargin}%
  }%
  \item\relax}
 {\endlist}
\begin{document}
\renewcommand\floatpagefraction{.7}
\renewcommand\textfraction{.2}

\maketitle

\begin{abstract}
\label{sec:abstract}
Mixture modeling is a general technique for making any simple model
more expressive through weighted combination.
This generality and simplicity in part explains the success of the Expectation Maximization (EM) algorithm,
in which updates are easy to derive for a wide class of mixture models.
However, the likelihood of a mixture model is non-convex, so EM has no known global convergence guarantees.
Recently, method of moments approaches offer global guarantees for some mixture models,
but they do not extend easily to the range of mixture models that exist.
In this work, we present Polymom, an unifying framework based on method of moments
in which estimation procedures are easily derivable, just as in EM.
Polymom is applicable when the moments of a single mixture component
are polynomials of the parameters.
Our key observation is that the moments of the mixture model are
a mixture of these polynomials,
which allows us to cast estimation as a Generalized Moment Problem.
We solve its relaxations using semidefinite optimization,
and then extract parameters using ideas from computer algebra.
This framework allows us to draw insights and apply tools from convex
optimization, computer algebra and the theory of moments to study
problems in statistical estimation.
Simulations show good empirical performance on several models.
\end{abstract}

\section{Introduction}
\label{sec:intro}

Mixture models play a central role in machine learning and statistics,
with diverse applications including bioinformatics, speech, natural language, and computer vision.
The idea of mixture modeling is to explain data through
a weighted combination of simple parametrized distributions \citep{titterington1985statistical,mclachlan2004finite}.
In practice, maximum likelihood estimation via Expectation Maximization (EM) has been the
workhorse for these models, as the parameter updates are often easily derivable.
However, EM is well-known to suffer from local optima.
The method of moments, 
dating back to Pearson \citep{pearson1894contributions} in 1894,
is enjoying a recent revival \citep{anandkumar12moments,anandkumar12lda,anandkumar13tensor,hsu12identifiability,hsu13spherical,chaganty13regression,kalai2010efficiently,hardt2014sharp,ge2015learning,balle2014spectral} due to
its strong global theoretical guarantees.
However, current methods depend strongly on the specific distributions 
and are not easily extensible to new ones. 

In this paper, we present a method of moments approach, which we call \polymom{},
for estimating a wider class of mixture models
in which the moment equations are polynomial equations (\sectionref{problem-formulation}).
Solving general polynomial equations is NP-hard, 
but our key insight is that for mixture models,
the moments equations are \emph{mixtures
  of polynomials equations} and we can hope to solve them if the moment equations
for each mixture component are \emph{simple} polynomials equations that we can solve.
\polymom{} proceeds as follows:
First, we recover mixtures of monomials of the parameters from
the data moments by solving an instance of the Generalized Moment
Problem (GMP) \citep{lasserre2011moments,lasserre2008semidefinite} (\sectionref{moment-completion}).
We show that for many mixture models, the GMP can be
solved with basic linear algebra and in the general case,
can be approximated by an SDP in which the moment equations are
linear constraints. 
Second, we extend multiplication matrix ideas from the computer algebra literature
\citep{stetter1993multivariate,moller1995multivariate,sturmfels2002solving,henrion2005detecting}
to extract the parameters by solving a generalized eigenvalue problem
(\sectionref{solution-extraction}).

\polymom{} improves on previous method of moments approaches in both generality and flexibility.
First, while tensor factorization has been the main driver for many of the method of moments
approaches for many types of mixture models,
\citep{anandkumar13tensor,anandkumar2013community,chaganty13regression,hsu13spherical,anandkumar2014provable,ge2015learning},
each model required specific adaptations which are non-trivial even for experts.
In contrast, \polymom{} provides a unified principle for tackling new models
that is as turnkey as computing gradients or EM updates.
To use \polymom{} (\figureref{diagram}), one only needs to provide a
list of observation functions ($\phi_n$) and
derive their expected values expressed symbolically as polynomials in the
parameters of the specified model ($f_n$).
\polymom{} then estimates expectations of $\phi_n$ and
outputs parameter estimates of the specified model.
Since \polymom{} works in an optimization framework,
we can easily incorporate constraints such as non-negativity and parameter tying
which is difficult to do in the tensor factorization paradigm.
In simulations, we compared \polymom{} with EM and tensor factorization
and found that \polymom{} performs similarly or better on some models
(\sectionref{models}). 
\newcommand{\catformat}[1]{\textbf{#1}}
\begin{center}
\begin{table}
\begin{tabularx}{\textwidth}{>{\centering\arraybackslash}X>{\centering\arraybackslash}X}
\begin{tabular}[t]{ll|}
\hhline{--}
\multicolumn{2}{|l|}{\catformat{mixture model}}\\ 
\hhline{--}
$\dat_t$  & data point ($\Re^D$) \\
$z_t$ & latent mixture component ($[K]$) \\
$\prm\ind{k}$ & parameters of component $k$ ($\Re^P$) \\
$\pi_k$  & mixing proportion of $p(z=k)$\\
$\allprm$  & all model parameters \\
\hhline{--}
\multicolumn{2}{|l|}{\catformat{moments of data}}\\ 
\hhline{--}
$\obs_n(\dat)$ & \obsname\ \\
$f_n(\prm)$ & \obsname\ \\
\hhline{--}
\multicolumn{2}{|l|}{\catformat{moments of parameters}}\\ 
\hhline{--}
$\Loy$ & the Riesz linear functional \\ 
$y_\balpha$ & $y_\balpha=\Loy(\prm^\balpha)$, $\balpha^\text{th}$
              moment\\
$\mu$ & probability measure for $\by$\\
$\by$ & $(y_\balpha)_\balpha$ the moment sequence\\
$\MM_{\mdeg}(\by)$ & moment matrix of degree $\mdeg$\\
\end{tabular}
\hspace{-7pt}
\begin{tabular}[t]{ll}
\hhline{--}
\multicolumn{2}{|l|}{\catformat{\catformat{sizes}}}\\ 
\hhline{--}
$D$& data dimensions \\
$K$& mixture components \\
$P$& parameters of mixture components \\
$T$& data points\\
$N$& constraints \\
$[N]$& $\{1, \dots, N\}$\\
$\mdeg$ & degree of the moment matrix\\
$s(\mdeg)$ & size of the degree $\mdeg$ moment matrix\\
\hhline{--}
\multicolumn{2}{|l|}{\catformat{polynomials}}\\ 
\hhline{--}
$\Re[\prm]$ & polynomial ring in variables $\prm$\\
$\NN$ & set of non-negative integers\\
$\balpha, \bbeta, \bgamma$ & vector of exponents (in $\NN^P$ or $\NN^D$) \\
$\prm^\balpha$ & monomial $\prod_{p=1}^P \sprm_p^{\balpha_p}$\\
$\coeff{n}{\balpha}$ & coefficient of $\prm^\balpha$ in $f_n(\prm)$\\
\end{tabular}
\end{tabularx}
\caption{
  \textbf{Notation}: 
We use lowercase letters (e.g., $d$) for indexing,
and the corresponding uppercase letter to denote the upper limit
(e.g., $D$, in ``sizes'').
We use lowercase letters (e.g., $\sprm\inds{k}{p}$) for scalars,
lowercase bold letters (e.g., $\prm$) for vectors,
and bold capital letters (e.g., $\bM$) for matrices.
}
\end{table}

\begin{figure}
  \includegraphics[width=\textwidth]{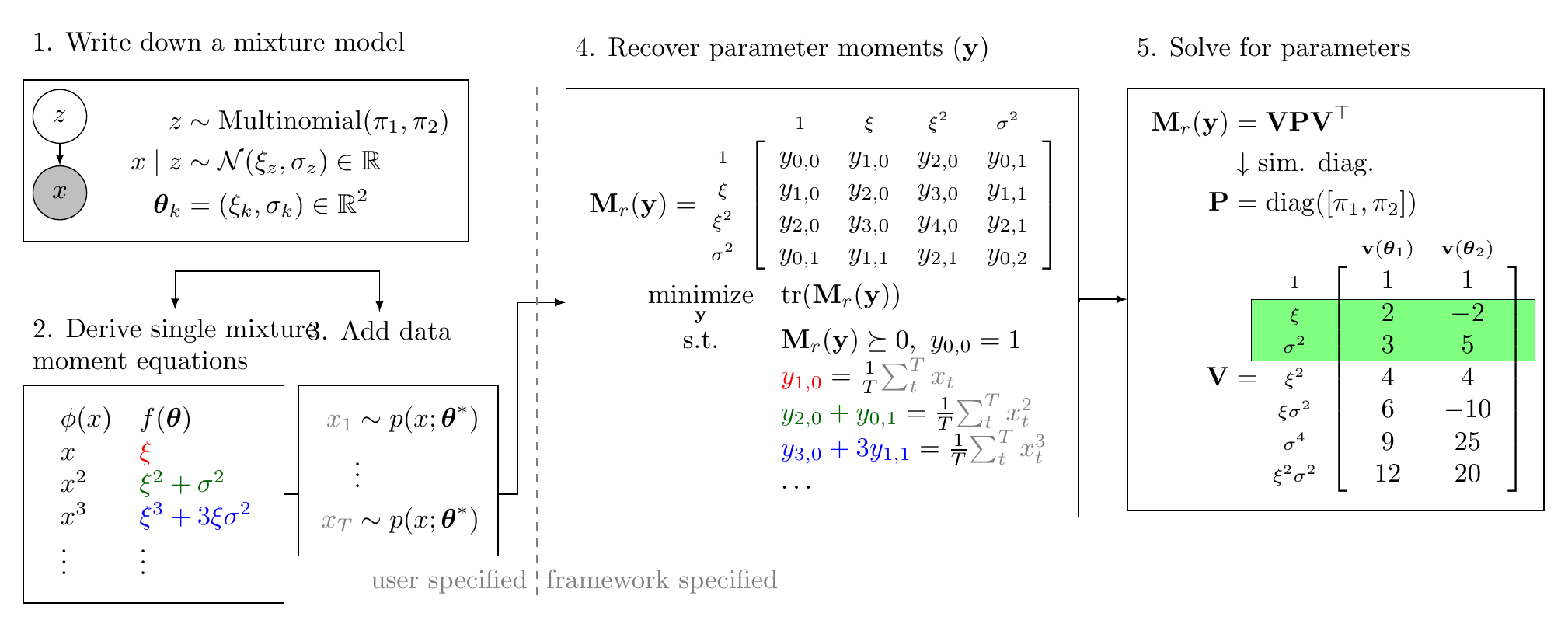}
  \caption{An overview of applying the \polymom{} framework.}
  \label{fig:diagram}
\end{figure}

\end{center}

\section{Problem formulation}
\label{sec:problem-formulation}


\subsection{The method of moments estimator}
\label{sec:statmom}

In a mixture model, each data point $\bx \in \Re^D$
is associated with a latent component $z \in [K]$:
\begin{align}
z \sim \text{Multinomial}(\pi), \quad \dat \ |\  z\sim p(\dat; \prmt\ind{z}),
\end{align}
where $\pi = (\pi_1, \dots, \pi_K)$ are the mixing coefficients,
$\prmt\ind{k} \in \Re^P$ are the true model parameters for the $k^{\textrm{th}}$ \emph{mixture component},
and $\dat \in \Re^D$ is the random variable representing data.
We restrict our attention to mixtures where each component distribution comes from the same parameterized family.
For example, for a mixture of Gaussians,
$\prmt\ind{k} = (\xi\ind{k}^* \in \Re^D, \Sigma\ind{k}^* \in \Re^{D \times D})$
consists of the mean and covariance of component $k$.

We define $N$ \emph{observation functions} $\obs_n : \Re^D \to \Re$
for $n \in [N]$ and define $f_n(\prm)$ to be the expectation of $\obs_n$ over a single component with parameters $\prm$,
which we assume is a simple polynomial:
\begin{align}
\label{eqn:momentFunction}
f_n(\prm) \eqdef \EE_{\dat \sim p(\dat; \prm)}[\obs_n(\dat)] = \sum_\balpha \coeff{n}{\balpha} \prm^\balpha,
\end{align}
where $\prm^\balpha = \prod_{p=1}^P \theta_p^{\alpha_p}$.
The expectation of each observation function $\EE[\obs_n(\dat)]$
can then be expressed as a mixture of polynomials of the true parameters 
\begin{align}
\label{eqn:momentExpectations}
\EE[\obs_n(\dat)]
=\sum_{k=1}^K \pi_k\EE[\obs_n(\dat) | z=k]
= \sum_{k=1}^K \pi_k f_n(\prmt\ind{k}).
\end{align}

The method of moments for mixtures seeks parameters $\allprm$ that satisfy the
\emph{moment conditions} expressed as the following polynomial equations:
\begin{align}
\label{eqn:momentEquations}
\EE[\obs_n(\dat)] = \sum_{k=1}^K \pi_k
f_n(\prm\ind{k}).
\end{align}%
where $\EE[\obs_n(\dat)]$ can be estimated from the data: $\frac1{T}\sum_{t=1}^T \obs_n(\dat_t)
\convergep \EE[\obs_n(\dat)]$.
Clearly, the true parameters $\allprmt$
satisfy these conditions as in \eqref{eqn:momentExpectations}. 
The goal of this work is to find parameters satisfying moment conditions
that can be written in the mixture of polynomial form
\eqref{eqn:momentEquations}.
We assume that the $N$ given observations functions
$\obs_1, \dots, \obs_N$ uniquely identify the model parameters
(up to permutation of the components).



\begin{example}[1-dimensional Gaussian mixture]
  \label{exa:mog}
Consider a $K$-mixture of 1D Gaussians with parameters $\prm\ind{k} = [\xi_k, \sigma_k^2]$
corresponding to the mean and variance, respectively, of the $k$-th component (\figureref{diagram}: steps 1 and 2).
We choose the observation functions,
$\obs(\sdat) = [\sdat^1, \ldots, \sdat^6],$
which have corresponding moment polynomials,
\[f(\prm) = [\xi, \xi^2 + \sigma^2, \xi^3 + 3 \xi \sigma^2, \dots].\]
For example, instantiating \eqref{eqn:momentEquations},
$\E[x^2] = \sum_{k=1}^K \pi_k (\xi_k^2 + \sigma^2_k)$.
Given $\obs(\dat)$ and $f(\prmt)$, and data, the \polymom{} framework can recover the parameters.
Note that the 6 moments we use have been shown by \citet{pearson1894contributions} to be sufficient for a mixture of two Gaussians.


\end{example}


\begin{example}[Mixture of linear regressions]
  \label{exa:mlr}
\providecommand{\mlrw}{w} 

Consider a mixture of linear regressions \citep{viele2002regression,chaganty13regression},
where each data point $\dat = [x, y]$ is drawn from component $k$ by sampling $x$ from an unknown distribution \emph{independent} of $k$ and setting $y = w_k x + \epsilon$, where $\epsilon \sim \sN(0, \sigma_k^2)$.
The parameters $\prm\ind{k} = (w_k, \sigma^2_k)$ are the slope and noise variance for each component $k$.
Let us take our observation functions to be \[\obs(\dat) = [x, xy,
xy^2, x^2, \ldots, x^3 y^2],\] for which the moment polynomials are 
\[f(\prm) = [\E[x], \E[x^2] w, \E[x^3] w^2 + \E[x] \sigma^2, \E[x^2], \ldots].\]
\end{example}

In \exampleref{mog}, the coefficients $\coeff{n}{\balpha}$ in the
polynomial $f_n(\prm)$ are just constants determined by integration.
For the conditional model in \exampleref{mlr}, the coefficients depends on the
data. While \exampleref{mlr} works, \polymom{} cannot handle arbitrary data
dependence, see \appendixref{separability} for sufficient conditions and counterexamples.

\subsection{Solving the moment conditions}
Our goal is to recover model parameters
$\prmt\ind{1}, \ldots,\prmt\ind{K} \in \Re^P$
for each of the $K$ components of the mixture
model that generated the data
as well as their respective mixing proportions $\pi_1, \ldots, \pi_K \in \Re$. 
To start, let's ignore sampling noise and identifiability issues and
suppose that we are given exact moment conditions as defined in \eqref{eqn:momentEquations}.
Each condition $f_n \in \Re[\prm]$ is a polynomial of the parameters $\prm$,
for $n = 1, \dots, N$.


\equationref{momentEquations} is a polynomial system of $N$ equations in the
$K+K\times P$ variables $[\pi_1, \ldots, \pi_K]$ and
$[\prm\ind{1}, \ldots, \prm\ind{K}] \in \Re^{P \times K}$.
It is natural to ask if standard
polynomial solving methods can solve \eqref{eqn:momentEquations} in the case
where each $f_n(\prm)$ is simple.
Unfortunately, the complexity of general polynomial equation solving is lower bounded by
the number of solutions, and each of the $K!$ permutations of the
mixture components corresponds to a distinct solution of
\eqref{eqn:momentEquations} under this polynomial system representation.
While several methods can take advantage of symmetries in polynomial
systems \citep{sturmfels2008algorithms, corless2009symmetries},
they still cannot be adapted to tractably solve
\eqref{eqn:momentEquations} to the best of our knowledge.



The key idea of \polymom{} is to exploit the mixture representation of the moment equations
\eqref{eqn:momentEquations}.
One idea is to seek a equivalent representation of the moment conditions expressed as
polynomial equations \eqref{eqn:momentEquations} that is invariant to permutations
of the $K$ components.
Specifically, let $\mu^*$ be a particular ``mixture'' 
over
the component parameters $\prmt\ind{1}, \dots, \prmt\ind{k}$
(i.e. $\mu^*$ is a probability measure).
Then we can express the moment conditions \eqref{eqn:momentEquations} in terms of $\mu^*$:
\begin{align}
\label{eqn:polytomeas}
\EE[\phi_n(\dat)] = \int f_n(\prm)\ \mu^*(d\prm), 
\text{ where }
\mu^*(\prm) = \sum_{k=1}^K \pi_k \delta(\prm-\prmt\ind{k}).
\end{align}
Conceptually, we no longer have any permutation invariance because the variable is $\mu$.
While permuted solutions of \eqref{eqn:momentEquations} are not equal
to each other in the parameter space,
$\mu$ remains the same measure regardless of the ``order in summing delta functions''.
As a result, solving the original moment conditions \eqref{eqn:momentEquations}
is equivalent to solving the following feasibility problem over $\mu$,
but where we deliberately ``forget'' the permutation of the
components by using $\mu$ to represent the problem:
\begin{align}
\label{prob:gmp}
\feasibility {\mu \in \Mea_+(\Re^P), \text{ the set of probability
  measures over } \Re^P} 
{\int f_n(\prm)\ \mu(d\prm) = \EE[\obs_n(\dat)], \ \ n=1,\ldots, N\\
&\mu\ \textrm{is}\ K\textrm{-atomic} \text{ (i.e. sum of $K$ deltas)}.}
\end{align}
If the true model parameters $\allprmt$ can be identified by the $N$ observed
moments up to permutation, then the measure
$\mu^*(\prm) = \sum_{k=1}^K \pi_k \delta(\prm-\prmt\ind{k})$
solving \problemref{gmp} is also unique.

\polymom{} solves \problemref{gmp} in two steps:
\begin{enumerate}
  \item Moment completion (\sectionref{moment-completion}):
    We show that \problemref{gmp} over the measure $\mu$
    can be relaxed to an SDP over a certain (parameter) \emph{moment matrix}
    $\MM_\mdeg(\by)$ whose optimal solution 
    is $\MM_\mdeg(\by^*) = \sum_{k=1}^K \pi_k \bv_\mdeg(\prmt\ind{k}) \bv_\mdeg(\prmt\ind{k})^\top$,
    where $\bv_\mdeg(\prmt\ind{k})$ is the vector of all monomials of degree at most $r$.
  \item Solution extraction (\sectionref{solution-extraction}):
    We then take $\MM_\mdeg(\by)$
    and construct a series of generalized eigendecomposition problems,
    whose eigenvalues yield $\allprmt$.
\end{enumerate}

\paragraph{Remark.} From this point on, distributions and moments refer
to $\mu^*$ which is over \emph{parameters}, not over the data.
All the structure about the data is captured in the moment conditions \eqref{eqn:momentEquations}.

\section{Moment completion}
\label{sec:moment-completion}

The first step is to reformulate \problemref{gmp}
as an instance of the Generalized Moment Problem (GMP) introduced by
\cite{lasserre2008semidefinite}.
A reference on the GMP, algorithms for solving GMPs,
and its various extensions is \cite{lasserre2011moments}.
We start by observing that \problemref{gmp} only depends on the integrals of monomials under the measure $\mu$:
for example, if $f_n(\prm) = 2 \sprm_1^3 - \sprm_1^2 \sprm_2$,
then we only need to know the integrals over the constituent monomials
($y_{3,0} \eqdef \int \sprm_1^3 \mu(d\prm)$ and
$y_{2,1} \eqdef \int \sprm_1^2 \sprm_2 \mu(d\prm)$)
in order to evaluate the integral over $f_n$.
This suggests that we can optimize over the (parameter) \emph{moment sequence}
$\by = (y_\balpha)_{\balpha \in \NN^P}$,
rather than the measure $\mu$ itself.
We say that the moment sequence $\by$ has a \emph{representing measure} $\mu$
if $y_\balpha = \int \prm^{\balpha} \ \mu(d\prm)$ for all $\balpha$,
but we do not assume that such a $\mu$ exists.
The \emph{Riesz linear functional} $\Loy : \Re[\prm] \rightarrow \Re$
is defined to be the linear map such that
$\Loy(\prm^\balpha) \eqdef y_\balpha$ and $\Loy(1) = 1$.
For example, 
$\Loy(2 \sprm_1^3 - \sprm_1^2 \sprm_2 + 3) = 2 y_{3,0} - y_{2,1} + 3$.
If $\by$ has a representing measure $\mu$,
then $\Loy$ simply maps polynomials $f$ to integrals of $f$ against $\mu$.
The key idea of the GMP approach is to 
convexify the problem by treating $\by$ as
free variables and then introduce constraints to guarantee that $\by$ has a representing measure.
First, let $\bv_\mdeg(\prm) \eqdef \left[ \prm^\balpha\ : \degree{\balpha} \leq \mdeg \right] \in \Re[\prm]^{s(\mdeg)}$
be the vector of all $s(\mdeg)$ monomials of degree no greater than $\mdeg$. 
Then, define the \emph{truncated moment matrix}
as 
\[\MM_\mdeg(\by) \eqdef \Loy(\bv_\mdeg(\prm) \bv_\mdeg(\prm)^\T),\]
where the linear
functional $\Loy$ is applied elementwise (see \exampleref{moment-matrix} below).
If $\by$ has a representing measure $\mu$, then $\MM_\mdeg(\by)$ is simply a (positive) integral over rank $1$ matrices
$\bv_\mdeg(\prm) \bv_\mdeg(\prm)^\T$ with respect to $\mu$,
so necessarily $\MM_\mdeg(\by) \succeq 0$ holds. 
Furthermore, by \theoremref{flat} \cite{curto1996solution},
for $\by$ to have a $K$-atomic representing measure,
it is sufficient that $\rank(\MM_\mdeg(\by))
=\rank(\MM_{\mdeg-1}(\by)) = K$.
So \problemref{gmp} is equivalent to
\begin{align}
\label{prob:gmpsdprank}
\feasibility {\by \in \Re^\NN \quad (\textrm{or equivalently, find}\ \MM(\by))}  {\sum_{\balpha} \coeff{n}{\balpha}
 y_\balpha = \EE[\phi_n(\dat)], \ \ n=1,\ldots, N\\
& \MM_\mdeg (\by) \succeq 0,\ y_{\boldsymbol{0}} = 1\\
& \rank(\MM_\mdeg(\by)) = K \text{ and } \rank(\MM_{\mdeg-1}(\by)) = K.}
\end{align}

Unfortunately, the rank constraints in \problemref{gmpsdprank} are not tractable.
We use the following relaxation to obtain our final (convex) optimization problem
\begin{align}
\label{prob:gmpsdp}
\minprob{\by}{\tr(\bC \MM_{\mdeg}(\by))}
{\sum_{\balpha} \coeff{n}{\balpha}
 y_\balpha = \EE[\phi_n(\dat)], \ \ n=1,\ldots, N\\
& \MM_\mdeg (\by) \succeq 0,\ y_{\boldsymbol{0}} = 1}
\end{align}
where $\bC \succ 0$ is a chosen scaling matrix. 
A common choice is $\bC = \bI_{s(r)}$ corresponding to minimizing the nuclear norm of the moment matrix, the usual convex relaxation for rank.
\appendixref{theory-moment-completion}
discusses some other choices of $\bC$ and more theory on \problemref{gmpsdp}.
However, for special cases like three-view mixture models, mixture of
linear regressions, etc. 
\problemref{gmpsdprank} can also be solved with basic linear
algebra, and there is no need to solve \problemref{gmpsdp} (see \sectionref{models}). 



\begin{example}[moment matrix for a 1-dimensional Gaussian mixture]
  \label{exa:moment-matrix}
Recall that the parameters $\prm = [\xi, \sigma^2]$ are the mean and variance of a one dimensional Gaussian.
Let us choose the
monomials $\bv_2(\prm) = [1, \xi, \xi^2, \sigma^2]$.
Step 4 for \figureref{diagram} shows the moment matrix when using $\mdeg
= 2$.
The moment matrix for $\mdeg = 2$ is then:
\begin{align}
\MM_{\mdeg=2}(\by) = 
\kbordermatrix{&
       1       &     \xi     & \xi^2   & \sigma^2 & \xi^3  &  \xi c \\
       1& y_{0,0} & y_{1,0} & y_{2,0} & y_{0,1} & y_{3,0} & y_{1,1}\\
\xi&       y_{1,0} & y_{2,0} & y_{3,0} & y_{1,1} & y_{4,0} & y_{2,1}\\
\xi^2 &    y_{2,0} & y_{3,0} & y_{4,0} & y_{2,1} & y_{5,0} & y_{3,1}\\
c     &    y_{0,1} & y_{1,1} & y_{2,1} & y_{0,2} & y_{3,1} & y_{1,2}\\
\xi^3 &    y_{3,0} & y_{4,0} & y_{5,0} & y_{3,1} & y_{6,0} & y_{4,1}\\
\xi c &    y_{1,1} & y_{2,1} & y_{3,1} & y_{1,2} & y_{4,1} & y_{2,2}
}
\end{align}
Each row and column of the moment matrix is labeled with a monomial
and entry $(i,j)$ is subscripted by the product of the monomials in
row $i$ and column $j$.
For $\obs_2(\sdat) = \sdat^2$, we have $f_2(\prm)=\xi^2 + c$, which leads to the linear constraint $y_{2,0} + y_{0,1} - \EE[\sdat^2]= 0$.
For $\obs_3(\sdat) = \sdat^3$, $f_3(\prm)=\xi^3 + 3\xi c$, leading to the constraint $y_{3,0} + 3 y_{1,1} - \EE[\sdat^3]= 0$.
\end{example}

\paragraph{Related work.}
Readers familiar with the sum of squares
and polynomial optimization literature~\citep{lasserre2001global, laurent2009sums, parrilo2003minimizing, parrilo2003semidefinite}
will note that \problemref{gmpsdp} is similar to the SDP relaxation of a
polynomial optimization problem. 
However, in typical polynomial optimization,
we are only interested in solutions $\prmt$ that actually
satisfy the given constraints,
whereas here we are interested in $K$ solutions $\allprmt$,
whose \emph{mixture} satisfies constraints corresponding to the moment
conditions \eqref{eqn:momentEquations}.
Within machine learning, 
generalized PCA has been formulated as a moment problem
\citep{ozay2010gpca} and the Hankel matrix (basically the moment matrix) has been used to learn
weighted automata \citep{balle2014spectral}.
While similar tools are used, the conceptual approach and the
problems considered are different.
For example, the moment matrix of this paper consists of unknown
moments of the model parameters, whereas
exisiting works considered moments of the data that are always directly observable.

\paragraph{Constraints.}
Constraints such as non-negativity (for parameters which represent probabilities or variances)
and parameter tying \citep{koller2009probabilistic} are
quite common in graphical models and are not easily addressed with existing method of moments approaches.
The GMP framework allows us to
incorporate some constraints using localizing matrices
\citep{curto2000truncated}.
Consider the case of a 2D mixture of Gaussians where the
mean parameters $\xi_1$, $\xi_2$ lies on the parabola $\xi_1 - \xi_2^2 =
0$ for all components. 
In this case, we just need to add constraints to \problemref{gmpsdp}:
$y_{(1,0)+\bbeta} - y_{(0,2)+\bbeta} = 0$ for all $\bbeta \in \NN^2$
up to degree $\degree{\bbeta} \leq 2 \mdeg-2$. 
Thus, we can handle
constraints during the estimation procedure rather than projecting back
onto the constraint set as a post-processing step. 
This is necessary for
models that only become identifiable by the observed moments
after constraints are taken into account. 
By incorporating these constraints into parameter estimation, we can possibly identify the model parameters with fewer moments.
We describe this method and its learning
implications in \appendixref{constraints}.


\paragraph{Guarantees and statistical efficiency.}
In some circumstances, e.g. in three-view mixture models or the
mixture of linear regressions, the constraints fully determine the
moment matrix; we consider these cases in \sectionref{models} and \appendixref{applications-details}.
While there are no general guarantee on \problemref{gmpsdp}, the flat
extension theorem tells us when the moment matrix corresponds to a
unique solution (more discussions in \appendixref{theory-moment-completion}):

\begin{theorem}[Flat extension theorem \cite{curto1996solution}]
  \label{thm:flat}
  Let $\by$ be the solution to \problemref{gmpsdp} for a particular $\mdeg$. If $\MM_{\mdeg }(\by) \succeq 0$ and $\rank(\MM_{\mdeg - 1}(\by)) = \rank(\MM_{\mdeg}(\by))$ then $\by$ is the optimal solution to \problemref{gmpsdprank} for $K = \rank(\MM_{\mdeg}(\by))$ and there exists a unique $K$-atomic supporting measure $\mu$ of $\MM_{\mdeg}(\by)$.
\end{theorem}

Recovering $\MM_{\mdeg}(\by)$ is linearly dependent on small
perturbations of the input~\cite{freund2004sensitivity}, suggesting
that the method has polynomial sample complexity for most models
where the moments concentrate at a polynomially rate.
In \appendixref{extensions},
we discuss a few other important considerations like noise robustness,
making \problemref{gmpsdp} more statistically efficient,
and some open problems.

\section{Solution extraction}
\label{sec:solution-extraction}

Having completed the (parameter) moment matrix $\MM_\mdeg(\by)$ (\sectionref{moment-completion}),
we now turn to the problem of extracting the model parameters $\allprmt$.
The solution extraction method we present is based
on ideas from solving multivariate polynomial systems
where the solutions are eigenvalues of certain multiplication
matrices
\citep{stetter1993multivariate,moller1995multivariate,corless1995singular,stetter2004numerical}.\footnote{
\citet{dreesen2012roots} is a short overview and
\citet{stetter2004numerical} is a comprehensive treatment including
numerical issues.}
The main advantage of the
solution extraction view is that higher-order moments and structure in
parameters are handled in the framework without model-specific effort.

Recall that the true moment matrix is
\[\MM_{\mdeg}(\byt) = \sum_{k=1}^K \pi_k \bv(\prmt\ind{k}){\bv(\prmt\ind{k})}^\T,\]
where $\monos(\prm) \eqdef [\prm^{\balpha_1}, \ldots, \prm^{\balpha_{s(\mdeg)}}] \in \Re{[\prm]}^{s(\mdeg)}$
contains all the monomials up to degree $r$.
We use $\prm = [\sprm_1, \ldots, \sprm_P]$ for variables and $\allprmt$ for the true solutions to these variables (note the boldface).
For example, $\sprmt\inds{k}{p} \eqdef {(\prmt\ind{k})}_{p}$ denotes the $p^{\text{th}}$ value
of the $k^{\text{th}}$ component, which corresponds to a solution for the variable $\sprm_p$.
Typically, $s(\mdeg) \gg K, P$ and
the elements of $\monos(\prm)$ are arranged in a degree ordering so that
$||\balpha_i||_1 \leq ||\balpha_j||_1$ for $i \leq j$.
We can also write $\MM_{\mdeg}(\byt) = \bV \bP \bV^\top$,
where $\bV \eqdef [\bv(\prmt\ind{1}), \ldots, \bv(\prmt\ind{K})] \in \Re^{s(\mdeg) \times K}$
is the \emph{canonical basis}
and $\bP \eqdef \diag(\pi_1, \dots, \pi_K)$ contains the mixing proportions.
At the high level, we want to factorize $\MM_r(\by^*)$ to get $\bV$, however
we cannot simply eigen-decompose $\MM_{\mdeg}(\byt)$ since $\bV$ is not orthogonal.
To overcome this challenge, we will exploit the internal structure of $\bV$ to construct several other matrices that share the same factors and perform simultaneous diagonalization.

Specifically, let $\bV[\bbeta_1; \dots; \bbeta_K] \in \Re^{K \times K}$ be a sub-matrix of $\bV$
with only the rows corresponding to monomials with exponents
$\bbeta_{1}, \ldots, \bbeta_{K} \in \NN^P$. Typically, $\bbeta_{1}, \ldots, \bbeta_{K}$ are just the first $K$ monomials in $\bv$.
Now consider the exponent $\bgamma_p \in \NN^P$ which is $1$ in position $p$ and $0$ elsewhere, 
corresponding to the monomial $\prm^{\bgamma_p} = \sprm_p$.
The key property of the canonical basis is that
multiplying each column $k$ by a monomial $\sprmt\inds{k}{p}$ just performs a ``shift''
to another set of rows:
\begin{align}
  \label{eqn:Vshift}
\bV\smono{\bbeta_{1}; \ldots; \bbeta_{K}} \, \bD_p =
\bV\smono{\bbeta_{1} + \bgamma_p; \ldots; \bbeta_{K} + \bgamma_p},
\quad\text{ where }
\bD_p \eqdef \diag(\sprmt\inds{1}{p}, \ldots, \sprmt\inds{K}{p}).
\end{align}
Note that $\bD_p$ contains the $p^{\text{th}}$ parameter for all $K$ mixture
components.



\begin{example}[Shifting the canonical basis]
\label{exa:cannonicalbasis}
Let $\prm = [\sprm_1, \sprm_2]$ and the true solutions be $\prmt\ind{1} = [2,3]$
and $\prmt\ind{2} = [-2,5]$.
To extract the solution for $\sprm_1$ (which are $(\sprmt_{1,1}, \sprmt_{2,1})$),
let $\bbeta_1 = (1,0), \bbeta_2 = (1,1)$, and $\bgamma_1 = (1,0)$.
\begin{align}
\bV = 
\kbordermatrix{
& \bv(\prm\ind{1})    &      \bv(\prm\ind{2})  \\
1 &      1       & 1 \\
\sprm_1&    2 & -2 \\
\sprm_2&    3 & 5 \\
\sprm_1^2&  4 & 4 \\
\sprm_1\sprm_2& 6 & -10\\
\sprm_2^2& 9 & 25 \\
\sprm_1^2 \sprm_2& 12 & 20\\
}
\quad
\quad
\quad
\underbrace{\kbordermatrix{
& \monos_1    &   \monos_2  \\
\sprm_1&    2 & -2 \\
\sprm_1\sprm_2& 6 & -10
}}_{\bV\smono{\bbeta_1; \bbeta_2}}
\underbrace{\begin{bmatrix}
2 & 0\\
0 & -2
\end{bmatrix}}_{\diag(\sprm\inds{1}{1}, \sprm\inds{2}{1})}
= \underbrace{\kbordermatrix{
& \monos_1    &   \monos_2  \\
\sprm_1^2&  4 & 4 \\
\sprm_1^2 \sprm_2& 12 & 20
}}_{\bV\smono{\bbeta_1 + \bgamma_1; \bbeta_2 + \bgamma_1}}
\end{align}
\end{example}

While \eqref{eqn:Vshift} reveals the structure of $\bV$,
we don't know $\bV$.
However, we recover its column space $\bU \in \Re^{s(\mdeg) \times K}$ from the moment matrix $\MM_{\mdeg}(y^*)$, for example with an SVD\@.
Thus, we can relate $\bU$ and $\bV$ by a linear transformation:
$\bV = \bU \bQ$, where $\bQ \in \Re^{K \times K}$ is some unknown invertible matrix.
\eqref{eqn:Vshift} can now be rewritten as:
\begin{align}
  \label{eqn:Ushift}
\bU\smono{\bbeta_{1}; \ldots; \bbeta_{K}} \bQ \, \bD_p =
\bU\smono{\bbeta_{1} + \bgamma_p; \ldots; \bbeta_{K} + \bgamma_p} \bQ, \quad p = 1, \dots, P,
\end{align}
which is a generalized eigenvalue problem where $\bD_p$ are the eigenvalues
and $\bQ$ are the eigenvectors.
Crucially, the eigenvalues, $\bD_p = \diag(\theta^*_{1,p}, \ldots, \theta^*_{K,p})$ give us solutions to our parameters.
Note that for any choice of $\bbeta_{1}, \ldots, \bbeta_{K}$ and $p \in [P]$, we have generalized eigenvalue problems that share eigenvectors $Q$, though their eigenvectors $D_p$ may differ.
Corresponding eigenvalues (and hence solutions) can be obtained by solving a simultaneous generalized eigenvalue problem, e.g., by using random projections like Algorithm B of~\citep{anandkumar12moments} or more robust \citep{kuleshov2015simultaneous} simutaneous diagonalization algorithms \citep{cardoso1996joint,bunse1993numerical,afsari2006simple}.

\begin{algorithm}
\begin{algorithmic}
\caption{Basic solution extraction}
\label{algo:basic}
\REQUIRE column space basis $\bU \in \Re^{s(\mdeg) \times K}$, $\bbeta_{1}, \ldots, \bbeta_{K} \in \NN^P$ so that $\rank\left(\bUbasis\right) = K$
\ENSURE Estimated solutions $\prmh\ind{1}, \ldots, \prmh\ind{K} \in \Re^P$ 
\vspace{0.5em}
\FOR{parameter dimensions $p = 1, \ldots, P$}
\STATE  $\bB_p \gets \bU\smono{\bgamma_p+[\bbeta_{1}, \ldots, \bbeta_{K}]}$ 
\STATE  $\bgamma_q \gets [\indicator{p=q}]_{p=1,\ldots, P}$
\ENDFOR\\
\vspace{0.5em}
Find $\bQ$:
\STATE \quad solve the simultaneous eigenvalue problems: $\bB_p \bQ = \bUbasis \bQ \bD_p$ for $p = 1,\ldots, P$

Find $\prmh\ind{k}$:

\STATE \quad Let $[\bq_1, \ldots, \bq_K] \eqdef \bQ$ for $\bq_k \in \Re^{K \times 1}$
\STATE \quad $\sprmh\inds{k}{p} \gets \frac{\boldsymbol{\rho}^\T \bB_p \bq_k}{\boldsymbol{\rho}^\T \bUbasis \bq_k}$ for $p=1,\ldots,P$, $k=1,\ldots,K$, and arbitrary $\boldsymbol{\rho}$
\end{algorithmic}
\end{algorithm}

We describe one approach to solve \eqref{eqn:Ushift} (\algorithmref{basic}),
which is similar to Algorithm B of~\citep{anandkumar12moments}. 
The idea is to take $P$ random weighted combinations of the equations \eqref{eqn:Ushift}
and solve the resulting (generalized) eigendecomposition problems.
Let $R \in \Re^{P \times P}$ be a random matrix whose entries are drawn from $\sN(0, 1)$.
A simple approach to find $\bQ$ is solving 
\[\bUbasis^{\inv}\left(\sum_{p=1}^P R_{q,p} \bUbasisshift{p} \right)\bQ = \bQ \bD_q\]
for each $q = 1, \dots, P$.
The resulting eigenvalues can be collected in $\Lambda \in \Re^{P \times K}$,
where $\Lambda_{q,k} = \bD_{q,k,k}$.
Note that by definition $\Lambda_{q,k} = \sum_{p=1}^P R_{q,p} \prmt\inds{k}{p}$, 
so we can simply invert to obtain
$[\prmt\ind{1}, \dots, \prmt\ind{K}] = R^{-1} \Lambda$. 
Although this simple approach does not have great numerical properties, these
eigenvalue problems are solvable if the
eigenvalues $[\lambda_{q,1}, \ldots, \lambda_{q,K}]$ are distinct for
all $q$,
which happens with
probability 1 as long as the parameters $\prmt_k$ are different from
each other.
In \appendixref{solution-extraction-connections}, we show how the tensor
decomposition algorithm from~\cite{anandkumar12moments} can be seen as solving
\eqref{eqn:Ushift} for a particular instantiation of $\bbeta_1, \ldots
\bbeta_K$.

\section{Applications}
\label{sec:models}

\providecommand{\mlrx}{x}
\providecommand{\mlry}{\upsilon}
\providecommand{\mlrw}{\bw}
\providecommand{\mvm}{\boldsymbol{\xi}}

\newcommand{\applicationrow}[4]{
    \multicolumn{2}{c}{\bf #1} \\ \midrule
    {\bf Model}  & {\bf Observation functions}\\
    \multirow{3}{0.45\textwidth}{%
      \footnotesize
        #2
    }
    & {%
      \footnotesize
    #3
    }\\
    &
    {\bf Moment polynomials} \\
    & {%
      \footnotesize
    #4
    }
}

\begin{table}
  \caption{Applications of the \polymom{} framework.
  See \appendixref{application-details} for more details.}
\label{tbl:examples}
\begin{center}
  \begin{tabular}{p{0.45\textwidth} p{0.45\textwidth}}
    \toprule

    \applicationrow{ 
      Mixture of linear regressions
    }{ 
        $\dat = [\mlrx, \mlry]$ is observed where
        $\mlrx \in \Re^D$ is drawn from an unspecified distribution and
        $\mlry \sim \sN(\mlrw \cdot \mlrx, \sigma^2 I)$, and  $\sigma^2$ is known.
        The parameters are $\prm_k = (\mlrw_k) \in \Re^{D}$.
    }{ 
      $\phi_{\balpha, b}(\dat) = \mlrx^\balpha \mlry^b$ for $0 \le |\balpha| \le 3, b \in [2]$.
    }{ 
      \parbox{0.45\textwidth}{
        $f_{\balpha, 1}(\prm) = \sum_{p=1}^P \E[\mlrx^{\balpha+\bgamma_p}] w_p$ \\
        $f_{\balpha, 2}(\prm) = \E[\mlrx^{\balpha}] \sigma^2 + \sum_{p,q=1}^P \E[\mlrx^{\balpha} x_p x_q] w_p w_q$,
        where the $\bgamma_p \in \NN^P$ is $1$ in position $p$ and $0$ elsewhere.
      }
    } \\ \midrule
    \applicationrow{ 
      Mixture of Gaussians
    }{ 
        $\dat \in \Re^D$ is observed where
        $\dat$ is drawn from a Gaussian with diagonal covariance: $\dat \sim \sN(\mvm, \diag(\bc))$.
        The parameters are $\prm_k = (\mvm_k, \bc_k) \in \Re^{D + D}$.
    }{ 
      $\phi_{\balpha}(\dat) = \dat^\balpha$ for $0 \le |\balpha| \le 4$.
    }{ 
      $f_{\balpha}(\prm) = \prod_{d=1}^D h_{\alpha_d}(\xi_d, c_d)$.\tablefootnote{
        $h_{\alpha}(\xi, c) = \sum_{i=0}^{\lfloor\alpha/2\rfloor} a_{\alpha,\alpha-2i} \xi^{\alpha-2i} c^{i}$ and 
        $a_{\alpha, i}$ be the absolute value of the coefficient of the degree $i$ term of the $\alpha^\textrm{th}$ (univariate) Hermite polynomial. 
        For example, the first few are 
        $h_{1}(\xi, c) = \xi$, $h_{2}(\xi, c) = \xi^2+c$,
        $h_{3}(\xi, c) = \xi^3+3\xi c$,
        $h_{4}(\xi, c) = \xi^4 + 6\xi^2 c + 3c^2$.
      }
    } \\ & \\ \midrule

    \applicationrow{ 
      Mixture of Binomials
    }{ 
        $\dat \in \NN$ is observed where
        $\dat$ is drawn from a binomial of $m$ trials, $B(m,\sprm)$.
        The parameters are $\prm_k = p_k \in \Re$.
    }{ 
      $\phi_{i}(\dat) = \indicator{\dat = i }$ for $\dat \in \NN$, $0 \le \dat \le m$.
    }{ 
      $f_{i}(\dat) = \binom{m}{i}p^i (1-p)^{m-i}$.
    } \\ & \\ \midrule

    \applicationrow{ 
      Multiview mixtures
    }{ 
      With 3 views,
      $\dat = [x\oft{1}, x\oft{2}, x\oft{3}]$ is observed where $x\oft{1}, x\oft{2}, x\oft{3} \in \Re^D$ and 
      $x\oft{\ell}$ is drawn from an unspecified distribution with mean $\mvm\oft{\ell}$ for $\ell \in [3]$.
      The parameters are $\prmt_k = (\mvm_k\oft{1}, \mvm_k\oft{2}, \mvm_k\oft{3}) \in \Re^{D + D + D}$. 
    }{ 
      $\phi_{ijk}(\dat) = x_i\oft{1} x_j\oft{2} x_k\oft{3}$ where $1 \le i, j, k \le D$.
    }{ 
      $f_{ijk}(\prm) = \xi_i\oft{1} \xi_j\oft{2} \xi_k\oft{3}$.
    } \\ & \\ & \\ \bottomrule 
\end{tabular}
\end{center}
\end{table}

Let us now look at some applications of \polymom{}.
\tableref{examples} presents several models with corresponding observation functions and moment polynomials.
It is fairly straightforward to write down observation functions for a
given model.
The moment polynomials can then be derived by computing expectations under the model,
a computation comparable to deriving updates for EM\@. 

We implemented \polymom{} for several mixture models in Python and the
code can be found at \url{https://github.com/sidaw/polymom}. A simpler
and cleaner demostration of solving a mixture of Gaussian in the
noiseless case can be
found at \url{https://github.com/sidaw/mompy} in the form of an
IPython Notebook (\texttt{extra\_examples.ipynb}).
We used CVXOPT to handle the SDP and the random projections algorithm to extract solutions.
In \tableref{experiments}, we show the relative error
$\max_k  || \prm\ind{k} - \prmt\ind{k} ||_2 / ||\prmt\ind{k} ||_2$
averaged over 10 random models of each class.
\newcommand{\first}[1]{\textbf{#1}}
\newcommand{\sketchy}[1]{\sout{#1}}
\begin{table}
\centering
\begin{tabular}{|l|l|lll|lll|lll|}
   &         \textbf{Methd.}  &  EM&TF&Poly&       EM& TF&Poly&       EM&TF&Poly\\
\hhline{-----------}
\textbf{Gaussians} &$K,D$& \multicolumn{3}{c|}{$T=10^3$} & \multicolumn{3}{c|}{$T=10^4$} & \multicolumn{3}{c|}{$T=10^5$} \\
\hhline{-----------}
spherical  & $2,2$ &
\first{0.37} & 2.05 & 0.58 &
\first{0.24} & 0.73 & 0.29 &
0.19 & 0.36 & \first{0.14} \\
diagonal&  $2,2$  &        
\first{0.44} & \sketchy{2.15} & 0.48 &
0.48 & \sketchy{4.03} & \first{0.40} &
0.38 & \sketchy{2.46} & \first{0.35} \\
constrained &   $2,2$     &
0.49 & \sketchy{7.52} & \first{0.38} &
0.47 & \sketchy{2.56} & \first{0.30} &
0.34 & \sketchy{3.02} & \first{0.29} \\
\hhline{===========}
\textbf{Others} & $K,D$& \multicolumn{3}{c|}{$T=10^4$} & \multicolumn{3}{c|}{$T=10^5$} & \multicolumn{3}{c|}{$T=10^6$} \\
\hhline{-----------}
3-view&  $3,3$ &
\first{0.38}&0.51&0.57&          
0.31&0.33&\first{0.26}&              
0.36&0.16&\first{0.12} \\
lin. reg.&     $2,2$ &     
-&-&\first{3.51}&                
-&-&\first{2.60}&                 
-&-&\first{2.52} \\
\end{tabular}
\caption{$T$ is the number of samples, and the error metric is defined above.
\textbf{Methods:}
EM\@: sklearn initialized with k-means using 5 random restarts;
TF\@: tensor power method implemented in Python;
Poly\@: Polymom by solving \problemref{gmpsdp}.
\textbf{Models:} for mixture of Gaussians, we have
$\sigma \approx 2 || \mu_1 - \mu_2 ||_2$; `spherical' and `diagonal'
describes the type of covariance matrix.
The mean parameters of constrained Gaussians satisfies $\mu_1 + \mu_2 = 1$.
The best result is \first{bolded}. 
TF only handles spherical variance, but it was of interest to
see what TF does if the data is drawn from mixture of Gaussians with
diagonal covariance, these results are in \sketchy{strikeout}.
}
\label{tbl:experiments}
\end{table}

\paragraph{Guarantees.}
In the rest of this section, we will discuss guarantees on parameter recovery
for each of these models.  In summary, we match many of the existing
results in the literature for the mixture of linear regressions and
multiview mixtures when $K \le D$. In these case the moment matrix is fully determined
by the linear constraints and
\problemref{gmpsdp} is just a linear solve. More discussions can be
found in \appendixref{application-details}.

In addition, we can obtain \emph{per-instance} guarantees in the following sense.
Recall that \polymom{} involves solving an SDP relaxation and performing solution extraction.
If the SDP solution has a flat extension (\theoremref{flat}) at the true number of components $K$
(a checkable assumption),
then we have solved the moment completion problem exactly,
and since solution extraction always works, we are guaranteed to obtain the true parameters.
On the other hand, if the SDP solution has a higher rank $K' > K$,
then as a consolation prize,
we have found a $K'$-mixture model that matches the moments (that we observed)
of the true $K$-mixture model.

\section{Conclusion}
\label{sec:discussion}
We presented an unifying framework for learning many types of mixture models via
the method of moments.
For example, for the mixture of Gaussians, we can apply the same algorithm
to both mixtures in 1D needing higher-order moments~\citep{pearson1894contributions,hardt2014sharp}
and mixtures in high dimensions where lower-order moments
suffice~\cite{anandkumar13tensor}.
The Generalized Moment Problem \citep{lasserre2008semidefinite,lasserre2011moments} and
its semidefinite relaxation hierarchies is what gives us the
generality, although we rely heavily on the ability 
of nuclear norm minimization to recover the underlying rank.
As a result, while we always obtain parameters satisfying the moment conditions,
we do not have formal guarantees on consistent estimation in general,
although we do have guarantees for several model families.
The second main tool is solution extraction,
which characterizes 
a more general structure of mixture models compared the tensor
structure observed by
\cite{anandkumar13tensor, anandkumar12moments}.
This view draws connections to the literature on solving polynomial
systems, where many techniques might be useful~\cite{stetter2004numerical,
  sturmfels2002solving, henrion2005detecting}.
Finally, through the connections we've drawn,
it is our hope that \polymom{} can make the method of
moments as turnkey as EM on more latent-variable models,
and provide a way to improve the statistical efficiency of method of moments procedures.



\paragraph{Acknowledgments.}
This work was supported by a Microsoft Faculty Research Fellowship to the third
author and a NSERC PGS-D fellowship for the first author.


\bibliographystyle{plainnat}
\bibliography{refdb/all}
\clearpage
\appendix
Some details and discussions are deferrred to the appendices.
\appendixref{applications-details} contains more details on the examples described in
\tableref{examples}; \appendixref{separability} defines separable models, which is the class of
models where
\polymom{} can be used, and gives a non-example;
\appendixref{theory-moment-completion} has more details and pointers
to references on the theory of moment completion;
\appendixref{extensions} describes some extensions such as constraints on
parameters, noise and some preliminary works on improving the statistical efficiency. 
\section{Examples}
\label{sec:applications-details}
\providecommand{\mprm}{\boldsymbol\xi} 
\providecommand{\minds}[2]{_{#2,#1}} 
\providecommand{\mind}[1]{_{#1}} 
\renewcommand{\bZ}{Z}
In this section, we first describe how undercomplete tensor factorization can be
seen as a special case of the solution extraction framework, and
elaborate on the mixture of Gaussians, the mixture of linear regressions
and the multiview mixture model.

\subsection{Tensor factorization as solution extraction}
\label{sec:solution-extraction-connections}

\begin{example}[Tensor decomposition as solution extraction]
Many latent variable models have been tackled via tensor decomposition
\citep{anandkumar13tensor}, and symmetric, undercomplete tensor decomposition
can be framed as a solution extraction problem.
Suppose we observe the tensor $\bT \eqdef \sum_{k=1}^K {\prmt}_k^{\otimes 3} \in \Re^{P \times P \times P}$.
We would like to recover the components $\prmt_k$.
For us, the inputs are constraints $\sprm_r \sprm_s  \sprm_t  - T_{rst}=0$ for all $r,s,t=1,\ldots, P$.
Choose $\monos(\prm) = [1, \sprm_1, \ldots, \sprm_P, \sprm_1^2, \sprm_1 \sprm_2,
\ldots, \sprm_P^2] = [1, \prm, \vecs(\prm \otimes \prm)]$, where $\vecs: \Re^{P
\times P} \rightarrow \Re^{P^2}$ just flattens the matrix. 
In the simplest case, suppose $P=K$ and $\rank(\bU) = K$.
Then the fully observed $\bU$ is
\begin{align}
\bU = 
\kbordermatrix{
\textrm{size} & P\\
1      & \bU_1 \\
P  &    \bU_2 \\
P^2  &    \bU_3} = 
\kbordermatrix{
\textrm{terms} & \prm\\
1      & \Loy(\prm) \\
\prm   &    \Loy(\prm \otimes \prm) \\
\vecs(\prm \otimes \prm)   &   \Loy(\vecs(\prm \otimes \prm) \otimes \prm)  }
\end{align}
where the linear functional $\Loy$ applies elementwise.
One choice of basis is just all the variables $\bUbasis=\bU_2$ and the eigenvalue problem we are required to solve is the generalized Hermitian eigenvalue problem $\bU_2 \bQ \bD = \left(\sum_{p=1}^P \eta_p \Loy(\sprm_p \prm \otimes \prm ) \right) \bQ$.
\cite{anandkumar12moments} proposed an algorithm that is procedurally identical,
where, in their notation $\operatorname{Pairs} \eqdef \bU_2$ and $\operatorname{Triples}(\eta) \eqdef \left(\sum_{p=1}^P \eta_p \Loy(\sprm_p \prm \otimes \prm ) \right)$,
and the algorithm proposed needed to solve the eigenvalue problem $B(\eta) = \operatorname{Pairs}^{-1} \operatorname{Triples}(\eta)$.
\end{example}

Typically, $\bbeta_{1}, \ldots, \bbeta_{K}$ are just the first $K$ monomials in $\bv$
(i.e. the $K$ monomials of the smallest degree).

Under this formulation, generalization to the fully-observed overcomplete tensor decomposition
case $K \ge D = P$ is clear if we observe enough moments to have enough basis vectors such that
$\rank(\bUbasis) = K$:
\begin{prop}
If $K \leq 1+P+P^2+\cdots+P^{\mdeg} = \frac{P^{\mdeg+1}-1}{P-1}$,
then solution extraction succeeds if we observe moments up to order $2\mdeg+1$
and monomials vectors of the true parameters
$\monos_{\mdeg}(\prm\ind{1}), \ldots, \monos_{\mdeg}(\prm\ind{K})$
are linearly independent.
\end{prop}
\begin{proof}
 To get the theoretical result, it suffices to consider higher-order moments:
\begin{align}
\bU =
\kbordermatrix{
\textrm{terms} & \vecs(\prm^{\otimes \mdeg})\\
\vecs(\prm^{\otimes \mdeg})   & \Loy( \vecs(\prm^{\otimes \mdeg}) \otimes \vecs(\prm^{\otimes \mdeg}) )  \\
\vecs(\prm^{\otimes \mdeg+1})  &  \Loy( \vecs(\prm^{\otimes \mdeg+1}) \otimes \vecs(\prm^{\otimes \mdeg}) )}
\end{align}
where we can take the $\bUbasis$ from the top block, and
$\bUbasisshift{q}$
belongs to the bottom block for all $q$.
So $2\mdeg+1$ order moments is needed if $K \leq P^r$ and this result is comparable
to~\cite{anandkumar2014sample}.
In practice, we would take all moments $\vecs(\prm^{\otimes 1}), \dots, \vecs(\prm^{\otimes \mdeg+1})$.
We may use lower order moments as well:
\begin{align}
\bU =
\kbordermatrix{
\textrm{terms} & \vecs(\prm^{\otimes 1}) \quad \vecs(\prm^{\otimes 2}) \quad \cdots \quad \vecs(\prm^{\otimes \mdeg})\\
\vecs(\prm^{\otimes 1})   & \vdots \\
\vecs(\prm^{\otimes 2})   & \\
\vdots & \cdots  \ \Loy(\vecs(\prm^{\otimes l}) \otimes
                             \vecs(\prm^{\otimes m})) \ \cdots \\
\vecs(\prm^{\otimes \mdeg+1})  & \vdots \\
}
\end{align}
where the entry of this matrix at block $l,m$ is
$\Loy(\vecs(\prm^{\otimes l}) \otimes \vecs(\prm^{\otimes m}))$ as
expected. While this still requires observing $2\mdeg+1^{th}$ order
moments, lower order moments are more accurate and can result in
better parameter estimates.
\end{proof}

\subsection{Moment completion for specific models}
\label{sec:application-details}

For several mixture models, we work out the polynomial
constraints, and then discuss the moment
completion problem. 
\subsubsection{Mixture of Linear Regressions}
\providecommand{\mlrx}{x}
\providecommand{\mlry}{\upsilon}
\providecommand{\mlrw}{w}
In \exampleref{mlr}, we described the mixture of linear regressions model in 1-dimension with parameters $\prmt_k = (w_k, \sigma^2_k)$.
Let us now consider the $D$-dimensional extension: we observe $\dat = [\mlrx, \mlry]$\footnote{
  We use $\mlry$ here since $y$ is reserved for the parameter moments.}
where $\mlrx \eqdef [x_1, \ldots, x_D]$ is drawn from an unspecified distribution and $\mlry = \mlrw \cdot \mlrx + \epsilon$ with $\epsilon \sim \sN(0, \sigma^2)$ for a known $\sigma$.
The parameters are $\prmt_k = (\mlrw_k)$ for $1 \le k \le K$. 
Next,
we choose observation functions $\obs_{\balpha, b}(\dat) = \mlrx^\balpha \mlry^b$ for $\balpha : 0 \le |\balpha| \le 3$ and $0 \le b \le 3$,
with corresponding moment polynomials:
$f_{\balpha, b}(\prm, \dat) = \mlrx^\balpha \EE_{\epsilon \sim
\sN(0,\sigma^2)} \left[{(\mlrw \cdot \mlrx + \epsilon)}^b\right]$.
These polynomials can be expressed in closed form using Hermite polynomials (see \sectionref{mog}).
For example, $f_{\zeros,2}(\prm, \dat) = \left({(\mlrw \cdot \mlrx)}^2 + \sigma^2\right)$.


Given these observation functions and moment polynomials, and data, the \polymom{} framework solves the moment completion problem (\problemref{gmpsdprank}) followed by solution extraction (\sectionref{solution-extraction}) to recover the parameters.
Further, we can guarantee that \polymom{} can recover parameters for this model when $K \le D$ by showing that \problemref{gmpsdprank} can be solved exactly.
Note that while no entry of the moment matrix is directly observed, each observation gives us a linear constraint on the entries of the moment matrix.
Let $\bgamma_p \in \NN^P$ be
the vector with value $1$ at position $p$ and $0$ elsewhere, then 
$\Loy(f_{\balpha, 1}(\prm)) = \sum_{p=1}^P
\E[\mlrx^{\balpha+\bgamma_p}] y_{\bgamma_p}$, and
$\Loy(f_{\balpha, 2}(\prm)) = 
\left(\E[\mlrx^{\balpha}] \sigma^2 + \sum_{p,q=1}^P 
\E[\mlrx^{\balpha+\bgamma_p+\bgamma_q}] y_{\bgamma_p+\bgamma_q}\right)$,
etc. 
When $K \le D$, there are enough equations that this system admits an unique solution for $\by$. 

Note that~\cite{chaganty13regression} recover parameters for this model by solving a series of low-rank tensor recovery problems, which ultimately requires the computation of the same moments described above.
In contrast, the \polymom{} framework makes the dependence on moments upfront and takes care of the heavy-lifting in a problem-agnostic manner.
Furthermore, we can even obtain parameters outside the regime of~\cite{chaganty13regression}: with the above observation functions and moment polynomials, we can recover parameters (with a certificate) \ac{when $K=4, D=3$}.

\subsubsection{Mixture of Gaussians}
\label{sec:mog}
We now look at $D$-dimensional extensions to \exampleref{mog}. 
Let the data be drawn from Gaussians with diagonal covariance, $x \vert z \sim \sN(\xi_z, \diag(c_z))$. 
The parameters of this model are $\prmt_k = (\xi_k, c_k) \in \Re^{2D}$. 
The observable functions are $\phi_\balpha(\dat) \eqdef \dat^\balpha$,
and the moment polynomials are $f_\balpha(\prm) = \EE[\dat^\balpha] = \prod_{d=1}^D h_{\alpha[d]}(\xi[d], c[d])$, where 
$h_{\alpha}(\xi, c) = \sum_{i=0}^{\lfloor\alpha/2\rfloor} a_{\alpha,\alpha-2i} \xi^{\alpha-2i} c^{i}$ and 
$a_{\alpha, i}$ be the absolute value of the coefficient of the degree $i$ term of the $\alpha^\textrm{th}$ (univariate) Hermite polynomial. 
The first few are 
$h_{1}(\xi, c) = \xi$, $h_{2}(\xi, c) = \xi^2+c$, $h_{3}(\xi, c) =
\xi^3+3\xi c$, $h_{4}(\xi, c) =
\xi^4 + 6\xi^2 c + 3c^2$.

Using this set of $\phi_{\balpha}$ and $f_{\balpha}$, \polymom{} will attempt to solve the SDP in \problemref{gmpsdp} and recover the parameters.
In this case however, the moment conditions are non-trivial and we cannot guarantee recovery of the true parameters.
However, \polymom{}
is guaranteed to recover parameters that match the moments and that minimizes nuclear norm. 

This full covariance case poses no conceptual trouble for \polymom{}.
In the case of full covariance, Isserlis’ theorem (or Wick’s theorem)
allows us to derive these polynomials and
\cite{triantafyllopoulos2002moments} provides an algorithm for
computing these polynomials. Toeplitz covariance or
other structured covariances with parameter sharing or constraints
are also conceptually handled under \polymom{}.

We can modify this model by introducing constraints: consider the case of 2D mixture
where the mean parameters for all components lies on a parabola $\xi_1 -
\xi_2^2 = 0$. 
In this case, we just need to add constraints to \problemref{gmpsdp}:
$y_{(1,0)+\bbeta} - y_{(0,2)+\bbeta} = 0$ for all $\bbeta \in \NN^2$
up to degree $\degree{\bbeta} \leq 2 \mdeg-2$. 
\ac{A few more details.}

By incorporating these contraints at estimation time, we can possibly identify the model parameters with less moments.
See \sectionref{extensions} for more details.

\subsubsection{Mixture of Binomials}
We include a quick example on the mixture of binomials in 1 dimension
to illustrate how \polymom{} can be applied to a discrete model. In
this model, $x \in \NN$ and $0 \leq
x \leq m$ and each component is a binomial distribution for $m$ trials
each with probabiliy $p$ of success. The probability mass function for the entire
mixture model is $p(x) = \sum_{k=1}^K
\pi_k \binom{m}{x} p_k^x (1-p_k)^{m-x}$.
There are only $K$ scalar parameters $p_1, \ldots, p_K$ and the
observation function is just the empirical probabilities
$\phi_{i}(x) = \indicator{x = i }$ for $x,i \in \NN$, $0 \le x,i \le
n$, with corresponding polynomials $f_{i}(p) = \binom{m}{i}p^i
(1-p)^{m-i}$, which can be expanded to become linear constraints in
\problemref{gmpsdp}.

\subsubsection{Multiview Mixtures}
\label{sub:multiviewmix}
\providecommand{\mprm}{\boldsymbol\xi} 
\providecommand{\minds}[2]{_{#2,#1}} 
\providecommand{\mind}[1]{_{#1}} 
\renewcommand{\bZ}{Z}

\ac{Clean up!}
Here we consider the three-view mixture model which has been well studied in~\cite[section 3.3]{anandkumar13tensor}.
We will show that we can solve the model without explicit whitening, a transformation that has been shown to introduce noise\cite{kuleshov2015tensor}.
The model is a mixture of three conditionally independent arbitrary distributions parameterized by their conditional means: we have $z \sim \Multinomial{\pi}, \bx_l \vert z \sim p_{l}(\mprm^{(l)}_z)$ where $p_{l}(\mprm^{(l)}_z)$ is such that $E_{\bx_l \vert z}[\bx_l] = \mprm$. The parameters are $\prm\ind{k} = [\mprm^{(1)}, \mprm^{(2)}, \mprm^{(3)}]$. 
Using the observation functions $\phi = [x\oft{1}, x\oft{2}, x\oft{3}, x\oft{1} \otimes x\oft{2}, \ldots, x\oft{1} \otimes x\oft{2} \otimes x\oft{3}]$, we have the following moment polynomials, $f = [\mprm\oft{1}, \mprm\oft{2}, \mprm\oft{3}, \mprm\oft{1} \otimes \mprm\oft{2}, \ldots, \mprm\oft{1} \otimes \mprm\oft{2} \otimes \mprm{3}]$.

The multiview mixture model is another model for which we can guarantee parameter recovery when $K \le D$. To prove this is the case, we will again show that \problemref{gmpsdp} can be solved exactly.
It suffices to consider just the first $P$ columns of the
moment matrix $\MM_2$, which are almost directly observable. 
As before, $\vecs(\cdot)$ just flattens
a matrix into a vector.

\begin{align}
  \MM_{2}^\T = 
\kbordermatrix{
&  \mprm\mind{1}&  \mprm\mind{2}&  \mprm\mind{3}&  \vecs(\mprm\mind{1} \otimes \mprm\mind{2} )&  \vecs(\mprm\mind{1} \otimes \mprm\mind{3} )&  \vecs(\mprm\mind{2} \otimes \mprm\mind{3} )  \\ 
\mprm\mind{1}&  \bZ_{2,0,0}&  \bY_{1,1,0}&  \bY_{1,0,1}&  \bZ_{2,1,0}&  \bZ_{2,0,1}&  \bY_{1,1,1}  \\ 
\mprm\mind{2}&  \bY_{1,1,0}&  \bZ_{0,2,0}&  \bY_{0,1,1}&  \bZ_{1,2,0}&  \bY_{1,1,1}&  \bZ_{0,2,1}  \\ 
\mprm\mind{3}&  \bY_{1,0,1}&  \bY_{0,1,1}&  \bZ_{0,0,2}&  \bY_{1,1,1}&  \bZ_{1,0,2}&  \bZ_{0,1,2}  \\ 
}
\end{align}
where $Y_{\alpha_1,\alpha_3,\alpha_3}$ and $Z_{\alpha_1,\alpha_3,\alpha_3}$ are both equal to $\Loy(\mprm\mind{1}^{\otimes \alpha_1} \otimes \mprm\mind{2}^{\otimes\alpha_2} \otimes \mprm\mind{3}^{\otimes \alpha_3})$, but are used to respectively denote observed and unknown variables.
However, this equation is only partially true as both sides contain
the same set of values but the precise arrangements depends on where
the minor matrix appears in the moment
matrix. We ignore this problem as it should be clear from the row and column labels.
In the undercomplete case, it is assumed that $\rank(\bU) = K \leq
\min(P_1, P_2, P3)$, thus we can easily
complete this matrix using simple linear algebra in the exact case by
repeatedly applying \autoref{thm:gammamatrix} below.
Generally, we may try to complete the moment matrix by solving \problemref{gmpsdp} from these partial observations, provided that optimizing with the nuclear norm recovers the true rank.
\ac{Experimentally, we have also been able to recover parameters with certificates when $K=4, D=3$.}


\begin{lem}[low rank completion of missing corner]
\label{thm:gammamatrix}
For any matrix
$\Gamma =
\begin{bmatrix}
A & B\\
C & X
\end{bmatrix}
$
with a missing block $X$,
where $\rank(\Gamma) = \rank(A) = \rank(B) = K$ and $A \in \Re^{K \times K}$, $X =
C A^{\inv} B$ uniquely completes $\Gamma$.
\begin{prf}
Because $A$ contains the entire $K$ elements basis, there exists unique $Y,Z \in \Re^{K \times K}$ so that
$B = AY$
and $C = ZA$. Similarly, $X = Z B = C A^{\inv} B$.
\end{prf}
\end{lem}

\section{Separability}
\label{sec:separability}
For conditional models, the coefficients of the moment polynomials can depend on the data but such dependence can sometimes break the process of converting from component moment constraints to mixture moment constraints.
In this section, we define separability, which is a sufficient condition on what dependence is allowed under \polymom{} and then we give some counterexamples.

Consider a mixture of linear regressions \citep{viele2002regression,chaganty13regression},
where the parameters $\prm\ind{k} = (w_k, \sigma^2_k)$ are the slope and noise variance for each component $k$.
Then each data point $\dat = [x, y]$ is drawn from component $k$
by sampling $x$ from an unknown distribution \emph{independent} of $k$
and setting $y = w_k x + \epsilon$, where $\epsilon \sim \sN(0, \sigma_k^2)$.
If we take observation function $\obs_{b,c}(\dat) = x^b y^c$,
then the corresponding $f_{b,c}(\prm)$ depends on the unknown distribution of $x$:
for example, $f_{1,2}(\prm) = \EE[x^3] w + \E[x] \sigma^2$.
In contrast, for the mixture of Gaussians, we had $f_2(\prm) = \mu^2 + \sigma^2$,
which only depends on the parameters.

However, not all is lost, since the key thing is that $f_{1,2}(\prm)$ depends
only on the distribution of $x$, which is independent of the component $k$
and furthermore can be estimated from data.
More generally, we allow $f_n$ to depend on $\dat$ but in a restricted way.
We say that $f_n(\prm, \dat)$ is \emph{separable} if
$\E[f_n(\prm, \dat)]$ does not depend on the parameters $\allprm$ of the mixture generating $\dat$. In other words,
\begin{align}
\label{eqn:separable}
\EE[\obs_n(\dat)] = \EE[f_n(\prm, \dat)]
\text{ where for all }k: \
\E[f_n(\prm, \dat) \mid z=k] = \E[f_n(\prm, \dat)] \in \Re[\prm].
\end{align}
In this case, we can define $f_n(\prm) \eqdef \E[f_n(\prm, \dat)]$,
and \eqref{eqn:momentEquations} is still valid.
For the mixture of linear regressions,
we would define $f_{b,c}(\prm, \dat) = x^b \, \EE_{\epsilon \sim \sN(0, \sigma^2)}[(w x + \epsilon)^c]$.
In this more general setup,
the approximate moment equations on $T$ data points
is $\frac1T \sum_{t=1}^T [f_n(\prm, \dat_t)] = \frac1T \sum_{t=1}^T \obs_n(\dat_t)$.

\providecommand{\mlrw}{w} 
An example of non-separability is a
mixture of linear regressions where the variance is not a parameter
and is different across mixture components:
$\prm = (w)$ and $\dat = (x, y)$.
Recall that $\E[x y^2] = \sum_{k=1}^K \pi_k (\E[x^3] w_k^2 + \E[x] \sigma_k^2)$,
but $\E[x^3] w_k^2 + \E[x] \sigma_k^2$ cannot be written as $\E[f_n(w_k, \dat)]$ for any $f_n$,
since it depends on $\sigma_k^2$. 
Thus, this example falls outside our framework.
In the simplest case, we can make $f_n(w, \dat)$ separable by introducing $\sigma_k$ as a parameter,
but this is not always possible if the noise distribution is unknown or if $\sigma_k(x)$ depends on $x$.
For example, if we have heteroskedastic noise, $\EE[x^a (y - w \cdot x) ] = 0$ are valid moment constraints for individual components, but it is not clear how to convert this to the mixture case.

\section{Theory of the moment completion problem}
\label{sec:theory-moment-completion}
For solution extraction, we assumed that moments of all monomials are
observed but for many models only polynomials of parameters can be
estimated from the data.
For example, in a Gaussian mixture the 2$^{nd}$ moment observable function $\obs(x) =
\xi^2 + c$ is a polynomial, but solution extraction requires moments
of monomials like $\xi^2$ and $c$.
Furthermore, we assumed
in \sectionref{solution-extraction}
that there exists underlying true parameters $[\prmt\ind{k}]_{k=1}^K$
while an arbitrary moment sequence of the parameters $\by$ and its
corresponding moment matrix $\MM(\by)$ may not correspond to any
parameters (i.e. no representing measure). In \sectionref{models},
we showed how moment completion can be done with just
linear algebra for multiview models, and we now focus on the harder case
of having to solve the SDP \problemref{gmpsdp}.
\pled{pop back up and remind reader where we are in the context of estimating mixture models;
  refer back to examples;
  intuitively, moments give us polynomials $\xi_1^2 + c_1$, but solution extraction 
  needs monomials $\xi_i^2$
}

While we do not have a complete answer since the rank constrained
\problemref{gmpsdprank} cannot be solved,
we point to the relevant literature and give some sufficient conditions for
solution extraction and sufficient conditions for parameter recovery.

\subsection{Conditions for solution extraction}
In \sectionref{solution-extraction}, we showed that
simple conditions based only on the column space basis is sufficient for solution
extraction to be successful. However, to further investigate
consistency and noise, we need to address a few more important issues.
First consider the noiseless setting,
we may not have enough moment contraints to guarantee a unique solution (identifiability).
Even if we assume that we have enough constraints for identifying a $K$ mixture,
we still do not know if solving the relaxed \problemref{gmpsdp} that relaxed the $\rank = K$ constraint can recover the true parameters.
Second, under noise, there may not exist a rank $K$ basis of the moment matrix
and even when a rank $K$ basis exists, it may not correspond to any true parameters.

In the case when some moment matching parameters can be extracted, the
moment matrix satisfies the \emph{flat extension} condition, which is the
same as conditions in \sectionref{solution-extraction} where ``$\bB_p \eqdef \bU\smono{\bgamma_p + [\bbeta_1, \ldots, \bbeta_K]}$ is
observed'' and $\bUbasis$ is a column space basis of $\MM_{\mdeg}(\by)$.
Let the highest degree monomial
of $\bUbasis$ be of degree $\mdeg-1 = \deg(\prm^{\bbeta_K}) =
\degree{\bbeta_K}$,
and the highest degree monomial of $\bB_p \eqdef
\bU\smono{\bgamma_p + [\bbeta_1, \ldots, \bbeta_K]}$ be of degree
$\mdeg = \degree{\bgamma_p+\bbeta_K} =
\degree{\deg(\bbeta_K)}+1$. Since $\bUbasis$ is a basis of $\colsp(\MM_{\mdeg}(\by))$
\begin{align}
\rank\left(\MM_{\mdeg-1}(\by)\right) &= \rank\left(\bUbasis\right) = K\\
 &= \rank\left(\MM_{\mdeg}(\by)\right) \geq \rank\left(\bU\smono{\bgamma_p + [\bbeta_1, \ldots, \bbeta_K]}\right).
\end{align}
If we got this basis from the moment matrix, then we say that the moment matrix $\MM_{\mdeg-1}(\by)$ corresponding to
$\bUbasis$ has a flat extension, because $\MM_{\mdeg-1}(\by)$ can be extended to a
moment matrix $\MM_{\mdeg}(\by)$ with higher degree monomials
without an increase in rank.
The concept of flat extension and its consequences are of central importance for the
 truncated moment problem, which is quite relevant to our problem and studied by
\cite{curto1996solution,curto1998flat,curto2000truncated, curto2005truncated}.
Next, we reproduce the simplest flat extension theorem:
\begin{thm}[\cite{curto1996solution}: flat extension theorem]
\label{thm:flatext}
Suppose $\MM_{\mdeg-1}(\by) \succeq 0$ and there exists $\MM_{\mdeg}(\by)$ 
so that $\rank(\MM_{\mdeg}(\by)) = \rank(\MM_{\mdeg-1}(\by))$
(i.e. a flat extension), then there exists an unique
$\rank(\MM_{\mdeg}(\by))$-atomic representing measure $\mu$ of $\MM_{\mdeg}(\by)$.
\end{thm}

Here the first column of $\MM_{\mdeg}(\by)$ contains every monomial of degree up to $\mdeg$ so that $\deg(\monos_\mdeg(\prm)) = \mdeg$.  However, several generalizations of the flat extension theorem are also useful for estimation of mixture models where sparse monomials are handled \citep{laurent2008sparse, laurent2009generalized} or
where constraints are handled \citep{curto2005truncated}.

The conceptual importance is that \autoref{thm:flatext} allows us to
work with just the moment matrix satisfying constraints from possibly
noisy observations, without assuming the moment matrix is generated by
some true parameters. Of course, it also provides a checkable
criterion for when solutions can be extracted
\citep{nie2013certifying}. We still do not know if solving
\problemref{gmpsdp} provides a flat extension in a finite number of steps.
\cite{nie2014optimality, nie2014truncated, nie2013linear} investigated
this issue very recently and showed that linear optimization over the
cone of moments have finite convergence under generic
conditions (theorem 4.2 of \cite{nie2013linear}).

Still, our issue is not fully resolved as
representing measures under linear
constraints may not be unique, and as a result even a flat moment matrix
may not correspond to the true parameters. 
For parameter fitting, we'd like to find the solution
with minimal rank or otherwise optimal in some way. We explore this
issue next but unfortunately we can only give some partial answers.
\begin{prop}[existence of $\bC$]
In the noiseless setting, there exist $\bC$ so that minimizing
\mbox{$\bC \bigcdot \MM_{\mdeg}(\by)) = \bc \cdot \by$} will give the right solution.
\begin{prf}
Let $\MM_{\mdeg}(\by) = \bU \boldsymbol{\Sigma} \bU^\T$ be the SVD with
$\bU \in \Re^{s(\mdeg) \times K}$ and $\boldsymbol{\Sigma} \in \Re^{K
  \times K}$. Let $\bU_{\perp} \in \Re^{s(\mdeg) \times (s(\mdeg)-K)}$
be the orthogonal compliment of $\bU$, then any $\bC = \bU_{\perp} \bD
\bU_{\perp}^\T$ suffices and $\bD \in
\Re^{(s(\mdeg)-K) \times (s(\mdeg)-K)}$ is an arbitrary diagonal
matrix with positive diagonal elements.
\end{prf}
\end{prop}

The convex iteration algorithm \cite{dattorro05convexoptimization}
is one way to reduce rank that sometimes works for us empirically,
where if the convex iteration algorithm converges to 0,
then the moment matrix has rank $K$.

\section{Extensions}
\label{sec:extensions}

\subsection{Constraints on parameters}
\label{sec:constraints}
Constraints on parameters is a common and important consideration in
applications. 
While constraints can often be addressed in maximum likelihood
or \emph{maximum a prioterior}  learning using EM \citep[see shared
parameters]{koller2009probabilistic}, it is less clear how to address
constraints under the \tensordecomp{} approach because of its reliance on
\tensorstruct{} and it is well-known that MME generally can give us
parameters outside of the parameter space even in the well-specified case.

\begin{example} Examples of constraints on parameters
\textbf{Some parameters are known}: Gaussian with sparse covariance matrix
where we already know that some dimensions are uncorrelated;
to solve a substitution cipher using an HMM, the transitions matrix is
a language model that is given.

\textbf{Parameters are tied}: transitions in an HMM might only depend on the
relatively difference between states if the states are ordered i.e. the transition matrix is Toeplitz.

\textbf{Polytope constraints}: some of the parameters might be
probabilities (e.g. multinomial distribution): 
\[\prm=[\pi_1, \ldots,
  \pi_P, \xi_1, \ldots],\ \pi_p \geq 0,\ \sum_{p=1}^P \pi_p =1\]

\textbf{Semialgebraic constraints}: For some polynomial $g \in
\Re[\prm]$, $g_i(\prmt\ind{k}) \geq 0, i = 1, \ldots, I$. This includes discrete sets $\prm_i \in \{0,1\}$ and ellipsoids.
\end{example}

The obvious attempt is to project to the feasible set after computing an
unconstrained estimation with MME.
But this approach has several serious issues.
First, some constrained models are only identifiable after the
constraints are taken into account, which happens when the model has a
lot of parameters and we cannot observe correspondingly more moments.
In this case, unconstrained estimation is useful only if we can characterize the entire subset of
the parameters space satisfying moment conditions, which is generally
not possible in the \tensordecomp{} approach.
Second, we need to determine what projection to use.
In the case of two equal parameters, if one estimate is much more
noisy than the other, it can be better to just ignore the more noisy
estimate than to project under the wrong metric (see \exampleref{inefficient}).
Third and strangely, even in the case when the first two issues are handled,
it was observed by \cite{cohen2013experiments} for probablities parameters, that
clipping to 0 is empirically inferior compared to heuristics like taking the
absolute value, which is not a projection.

Under the \polymom{} formulation, we can take constraints into account
during estimation.
The technique of localizing matrix \citep{curto2000truncated} in moment theory allows us to deal
with semialgebraic constraints. Of course, the computational
complexity increases if the constraints are themselves complicated and
high degree. Next, we define the localization matrix, give an example,
and then give a constrained version of the flat extension theorem.
\begin{example}[localizing matrix for an inequality constraint]
Let $\prm = [c, \xi]$, so that $\prm^\balpha = c^{\alpha_1}
\xi^{\alpha_2}$ and $\Loy(\prm^\balpha) = y_\balpha$, and chose the
monomials $\bv_2(\prm) = [1, c, \xi, c^2, c \xi, \xi^2]$. Suppose
that $c$ is the variance and we want to have constraint that 
$c -1 \geq 0$, then
\begin{align}
\MM_{1}((c-1)\by) = 
\kbordermatrix{
&        1       &       c     & \xi     \\
1&      y_{1,0}-1      & y_{2,0} - y_{1,0} & y_{1,1}-y_{1,0}\\ 
c&    y_{2,0}-y_{1,0}     & y_{3,0}-y_{2,0} & y_{2,1}-y_{2,0}\\
\xi&    y_{1,1}-y_{0,1}     & y_{2,1} - y_{1,1} & y_{1,2}-y_{0,2}
}
\end{align}
it is clear that a necessary condition for extracted solutions
to satisfy the constraint $c - 1 \geq 0$ is that
$\MM_{1}((c-1)\by) \succeq 0$ since 
$\boldsymbol{f}^\T \MM_{1}((c-1)\by) \boldsymbol{f} = \Loy(f(\prm)^2 (c-1)) \geq 0$.
\end{example}

\subsection{Noise and statistical efficiency}
In the presense of noise \problemref{gmpsdp} may not be feasible and
even if it was, it may not be ideal to exactly match noisy moments.
Furthermore, it is argued that higher order moments are too noisy to
be useful, but there are also more of them and they do contain more
information about the model parameters as long as we can model how
noisy they are. We consider the problem with slack $\epsilon$ and a
weighting matrix $\bW \succ 0 \in \Re^{N\times N}$ modelling how much noise is
present in each \cnstrfunname{}. 
This effect is fairly well-known, and here is a very
simple example which shows that even much more noisy measurements can
improve efficiency.
\begin{example}[efficient estimation]
\label{exa:inefficient}
Suppose $X \sim \sN([\xi, \xi], \diag[\sigma^2, c \sigma^2])$ and we would
like to estimate the mean parameter $\xi$ by matching moments.
Any estimators of the form
$\hat{\xi} = \frac1{T} \sum_{t=1}^T(\gamma \sdat\inds{t}{1} +(1-\gamma) \sdat\inds{t}{2} )$ are
consistent and has risk
\begin{align}
R &= \EE \left[ (\hat{\xi}  - \xi)^2 \right] = \EE\left[ \left(\gamma
   \sum_{t=1}^T \dat\inds{t}{2} - \gamma \xi  +(1-\gamma)  \sum_{t=1}^T
    \dat\inds{t}{2}  - (1-\gamma)\xi\right)^2\right]\\
   &= \EE \left[ \gamma^2 \left(\xi-\frac1{T} \sum_{t=1}^T \dat\inds{t}{1}\right)^2
   + (1-\gamma)^2 \left(\xi -\frac1{T} \sum_{t=1}^T \dat\inds{t}{2}\right)^2
     \right]\\
     &= \frac1{T}(\gamma^2 \sigma^2 + (1-\gamma)^2 c \sigma^2)
\end{align}
under the squared loss, and the efficient estimator would have $\gamma
= \frac{c-1}{c}$ and a risk of
$\frac{\sigma^2}{T}\frac{c^2-c+1}{c^2}$. For $c=10$, the risk for
efficient estimation is $0.91 \frac{\sigma^2}{T}$ whereas for
$\gamma=0.5$, the risk is $2.75 \frac{\sigma^2}{T}$.
\end{example}

This example suggests that a weighting matrix $W$ has the potential to
make use of higher order moments and also give better estimates. Consider
\begin{align}
\label{prob:gmpslack}
\minprob{\bg, \by}{ \bC \bigcdot \MM(\by) }
{
\cnstrfun_n = \sum_{\balpha} \coeff{n}{\balpha}y_\balpha - \EE[\obs_n (\dat)\\
& \cnstrfuns^\T \bW \cnstrfuns \leq \epsilon \\
&\MM(\by) \succeq 0.\\
}
\end{align}
In the simplest case when $\bW = \bI_{N}$, and $\epsilon = 0$, 
\problemref{gmpslack} is the same as \problemref{gmpsdp}.

\begin{align}
\label{prob:gmpslacksdp}
\minprob{\bg, \by}{ \bC \bigcdot \MM(\by) }
{
\cnstrfun_n = \sum_{\balpha} \coeff{n}{\balpha}y_\balpha - \EE[\obs_n (\dat)] \\
& \bW \bigcdot  \bF \leq \epsilon \\
&\MM(\by) \succeq 0\\
& \begin{bmatrix}
      1 &   \cnstrfun^\T \\
      \cnstrfun & \bF
    \end{bmatrix} \succeq 0
}
\end{align}

A good weighting matrix $\bW$ should put more weights on moment conditions
that can be estimated more precisely. The asymptotically
efficient weighting matrix suggested by the \emph{Generalized Method
  of Moments} \citep{hansen1982gmm} is
\begin{align}
\bW^{\inv} = \EE\left[\cnstrfuns(\allprm, \dat)  \cnstrfuns(\allprm, \dat)^\T\right]
\approx \frac1{T} \sum_{t=1}^T \cnstrfuns(\allprm, \dat)  \cnstrfuns(\allprm, \dat)^\T
\end{align}

\begin{theorem}[\GeMM{} is asymptotically efficient
  \citep{hansen1982gmm}]
Let $g_n(\prm, \XX) \eqdef \sum_k f_n(\prm\ind{k}) - h_n(\XX)$ so that $\EE[h_n(\XX)] =
\EE[\phi_n(\dat)]$. 
Let $W^{\inv}= \EE[\bg(\prm, \XX) \bg(\prm, \XX)^\T ] \approx \frac1{T}
\sum_{t=1}^T \bg(\prm, \XX\ind{t}) \bg(\prm, \XX\ind{t})^\T$
Iterative \GeMM{} is efficient with this weighting matrix $\bW$.
\end{theorem}



\end{document}